\documentclass[twoside,11pt]{article}

\usepackage{jmlr2e}

\usepackage{amsmath}
\usepackage{algorithm}
\usepackage{algorithmic}
\usepackage{url}
\usepackage{lscape}

\newtheorem{thm}{Theorem}
\newtheorem{lem}{Lemma}
\newtheorem{defn}{Definition}
\newcommand{\mbf}[1]{\boldsymbol{#1}}
\newcommand{\Tr}[1]{\mathrm{\,Tr\,}#1}

\ShortHeadings{Probable convexity and Correlated Topic Models}{Khoat Than and Tu Bao Ho}
\firstpageno{1}

\begin{document}

\title{Probable convexity \\and its application to Correlated Topic Models\thanks{This work was partialy done when K. Than was at JAIST.}}

\author{\name Khoat Than \email khoattq@soict.hust.edu.vn \\
       \addr Hanoi University of Science and Technology, 1 Dai Co Viet road, Hanoi, Vietnam
       \AND
       \name Tu Bao Ho \email bao@jaist.ac.jp \\
       \addr Japan Advanced Institute of Science and Technology, 1-1 Asahidai, Nomi, Ishikawa 923-1292, Japan
%       \AND
%       \addr John von Neumann Institute, Vietnam National University, HCM, Vietnam
       }

\editor{}

\maketitle

\begin{abstract}%   <- trailing '%' for backward compatibility of .sty file
Non-convex optimization problems often arise from probabilistic modeling, such as estimation of posterior distributions.  Non-convexity makes the problems intractable, and poses various obstacles for us to design efficient algorithms. In this work, we attack non-convexity by first introducing the concept of \emph{probable convexity} for  analyzing convexity of real functions in practice. We then use the new concept to analyze an inference problem in the \emph{Correlated Topic Model} (CTM) and related nonconjugate models. Contrary to the existing belief of intractability, we show that this inference problem is concave under certain conditions. One consequence of our analyses is a novel algorithm for learning CTM which is significantly more scalable and qualitative than existing methods. Finally, we highlight that stochastic gradient algorithms might be a practical choice to resolve efficiently non-convex problems. This finding might find beneficial in many contexts which are beyond probabilistic modeling.
\end{abstract}


\section{Introduction} \label{sec:Introduction}

Estimation of posterior distributions plays a central role when developing probabilistic graphical models. With conjugate priors, we are likely able to derive efficient sampling algorithms for estimation \citep{GriffithsS2004,PritchardSD2000population}. When nonconjugate priors are used, the estimation problem is much more difficult, as observed in the topic modeling literature by \citet{BleiL07,SalomatinYL09,PutthividhyaAN10,PutthividhyaAN09,AhmedX2007,{BleiL2006DTM}}. A popular approach is to cast estimation as an optimization problem. Nonetheless, the resulting problems are often non-convex. Non-convexity poses various obstacles for designing efficient algorithms, and  does not allow us to directly exploit the nice theory of convex optimization.

In this work, we introduce the concept of \emph{probable convexity} that aims at two targets: (1) to reveal how hard an optimization problem in practice is; (2) to support us smoothly employ efficient methods of convex optimization to deal with non-convex problems. In a perspective,  probable convexity of a  family  $\mathfrak{F}$ of real functions essentially says that most members of $\mathfrak{F}$ are convex.  With such families, in practice we probably rarely meet non-convex functions from $\mathfrak{F}$. We remark that in many situations of data analytics (e.g., posterior estimation in graphical models) we often have to deal with not only one but many members of a family at once. Hence some appearances of non-convex members may not affect significantly the overall result. Hence a direct employment of convex optimization is possible and  beneficial. In other words, we could do minimization efficiently for functions of $\mathfrak{F}$ in practice.

We next use the concept to investigate estimation of posterior distributions in the \textit{Correlated Topic Model} (CTM) \citep{BleiL07} and related nonconjugate models. In particular, we study the problem of a posteriori estimating theta (topic mixture) for a given document: $\mbf{\theta}^*= \arg \max_{\mbf{\theta}} \Pr(\mbf{\theta}| \mbf{d})$. This is an MAP problem and is intractable for many models in the worst case \citep{SontagR11}. We show that under certain conditions, the objective function of this MAP problem is in fact \emph{probably concave}, i.e., concave with high probability. This suggests that posterior estimation of theta may be tractable in practice. Similar results are obtained for  related nonconjugate topic models.

The cornerstone of our analyses of nonconjugate models is the logistic-normal function which originates from the logistic-normal distribution \citep{Aitchison1980logistic}. We show in this work that the logistic-normal function is probably concave under certain conditions. This result may be of interest elsewhere and beneficial in practical applications, because the logistic-normal distribution is used as an effective prior in many contexts including topic modeling \citep{BleiL07,SalomatinYL09,PutthividhyaAN10,PutthividhyaAN09,BleiL2006DTM,Miao+2012LAA} and grammar induction \citep{Cohen2009shared,Cohen2010grammars}.

As a consequence of our analysis, a novel algorithm for learning CTM is proposed. This algorithm is surprisingly simple in which posterior estimation of theta is done by the Online Frank-Wolfe (OFW) algorithm  \citep{Hazan2012OFW}. From empirical experiments we find that the new algorithm is significantly faster than existing ones, while maintaining or making better the quality of the learned models. This further suggests that even though  MAP inference for CTM is intractable in the worst case, most instances in practice may be  resolved efficiently. 

Finally, we find that stochastic gradient decent (SGD) might be a practical choice to resolve efficiently non-convex problems. SGDs such as OFW \citep{Hazan2012OFW} are originally introduced in the convex optimization literature. They are often very efficient and have many advantages over deterministic algorithms, especially in large-scale settings. However, to our best knowledge, no prior study has been made to investigate the role of SGDs for resolving non-convex problems. We argue that due to their stochastic nature, SGD algorithms might be able to jump out of local optima to reach closer to global ones. Hence SGDs seem to be more advantageous than traditional (deterministic) methods for non-convex problems. We complement this observation by the successful use of OFW to solve  posterior estimation of theta in CTM.

\textsc{Organization:}
We present the concept of probable convexity in Section~\ref{ch5-sec:concepts}.  Section~\ref{LN-concavity} presents our analysis of the logistic-normal function. The study of CTM and related nonconjugate models is presented in Section~\ref{ch5-sec:MAP-in-CTM}. The new algorithm for learning CTM and experimental results are discussed in Section~\ref{ch5-sec:learning-CTM}. We also investigate in this section how well SGDs could resolve non-convex problems by analyzing OFW. The final section is for further discussion and conclusion.

\textsc{Notation:}
Throughout the paper, we use the following conventions and notations. Bold faces denote vectors or matrices. $x_i$ denotes the $i^{th}$ element of vector $\mbf{x}$, and $A_{ij}$ denotes the element at row $i$ and column $j$ of matrix $\mbf{A}$. Notation $\mbf{A} \le 0$ means that matrix $\mbf{A}$ is \textit{negative semidefinite}.  For a given vector $\mbf{x} = (x_1, ..., x_V)^t$, we denote $\frac{1}{\mbf{x}} = (\frac{1}{x_1}, ..., \frac{1}{x_V})^t$ and $\log \tilde{\mbf{x}} = (\log \frac{x_1}{x_V}, ..., \log \frac{x_{V-1}}{x_V})^t$. $diag(\mbf{x})$ denotes the diagonal matrix whose diagonal entries are $x_1, ..., x_V$, respectively. More notations are: 

\noindent
\begin{tabular}{rl}
   $\mathcal{V}$: & vocabulary of $V$ terms, often written as $\{1, 2,...,V\}$. \\
   $\boldsymbol{d}$: & a document represented as a count vector of $V$ dimensions, \\
   & $\boldsymbol{d}= (d_1, d_2, ..., d_V)$ where $d_j$ is the frequency of term $j$. \\
  $\mathcal{C}$: & a corpus consisting of $M$ documents, $ \{\boldsymbol{d}_1, ..., \boldsymbol{d}_M\}$. \\
  $K$: & number of topics. \\
  $\boldsymbol{\beta}_k$: & a topic which is a distribution over the vocabulary $\mathcal{V}$. It is written as \\
   & $\boldsymbol{\beta}_k = (\beta_{k1},...,\beta_{kV})^t$, where $\beta_{kj} \ge 0, \sum_{j=1}^{V} \beta_{kj} =1$. \\
   $\mathbb{E}$: & the expectation of a random variable. \\
  $\Delta_K$: & the unit simplex  in the $K$-dimensional space, \\
  & $ \Delta_K=  \{\mbf{x} \in \mathbb{R}^K: \sum_{k=1}^K x_k = 1, x_j \ge 0, \forall j \}$. \\
  $\overline{\Delta}_K$: & the interior of $\Delta_K$, that is  $\overline{\Delta}_K =  \{\mbf{x} \in \mathbb{R}^K: \sum_{k=1}^K x_k = 1, x_j > 0, \forall j \}$. \\
  $\mbf{e}_i$: & the $i^{th}$ unit vector in the Euclidean space, i.e, $e_{ii} =1$ and $e_{ij} =0, \forall j \ne i$. \\
    $\exp x$: & denotes $e^x$. \\
  $\mathcal{N(\mbf{\mu}, \mbf{\Sigma})}$: & the multivariate Gaussian distribution with mean $\mbf{\mu}$ and covariance  $\mbf{\Sigma}$. \\
  $x \sim \mathcal{A}(\cdot)$: &  the random variable $x$ follows the distribution $\mathcal{A}(\cdot)$. \\
  $\Tr \mbf{A}$: &  the trace of matrix $\mbf{A}$. \\
  $\lambda_i(\mbf{A})$ & the $i^{th}$ largest eigenvalue of  matrix $\mbf{A}$. \\
  $\mathbb{S}^K$: & the set of all symmetric matrix of size $K \times K$. \\
  $\mathbb{S}^K_+$: &  the set of all positive definite matrices of $\mathbb{S}^K$. \\
  $\nabla f$ or $f'$: & the gradient (first-order derivative) of the given function $f$.  \\
  $f''$: & the Hessian matrix (second-order derivative) of the given function $f$.  \\
  $\det \mbf{A}$: & the determinant of the square matrix $\mbf{A}$.
\end{tabular}

\section{Probable convexity} \label{ch5-sec:concepts}

Let $\mathfrak{F}(x; a)$ be a family of real functions defined on a set $X \subset \mathbb{R}^K$, parameterized by $a$. Each value of $a$ determines a function $f(x;a)$ of $\mathfrak{F}(x; a)$. 

\begin{defn}[probable convexity]
Let $\mathfrak{F}(x; a)$ be a family of functions defined on a set $X \subset \mathbb{R}^K$, parameterized by $a$. Family $\mathfrak{F}(x; a)$ is said to be \emph{probably convex} if there exists a positive constant $p$ such that any element of $\mathfrak{F}(x; a)$ is convex on $X$ with probability at least $p$. Equivalently, $\mathfrak{F}(x; a)$ is said to be \emph{p-convex} if any element of $\mathfrak{F}(x; a)$ is convex on $X$ with probability at least $p$.
\end{defn}

By definition, a family of convex functions is probably convex with probability 1. The family $\mathfrak{F}(x; a,b,c)= \{a x^2 + bx+c: a, b, c \in \mathbb{R}\}$ is probably convex with probability 1/2, since convexity of this family is decided by the sign of $a$.

In a perspective, probable convexity of a family may refer to the proportion of convex members in that family. High $p$ implies that most members are convex on $X$. Family $\mathfrak{F}(x; a,b,c)$ reflects well this perspective.

In another perspective, $p$-convexity of $\mathfrak{F}$ may refer to the case that every member of $\mathfrak{F}$ is convex over a part of  $X$. High $p$ implies that the members of $\mathfrak{F}$ is convex over most of $X$. As an example, family $\mathfrak{F}(x; a) = \{x^4 - 6x^2 +ax: x \in [-10, 10], a \in \mathbb{R}\}$ is 0.9-convex, because each member is convex over $90\%$ of $[-10, 10]$.

\begin{defn}[almost sure convexity]
Let $\mathfrak{F}(x; a)$ be a family of functions defined on a set $X \subset \mathbb{R}^K$, parameterized by $a$. Family $\mathfrak{F}(x; a)$ is said to be \emph{almost surely convex} if any element of $\mathfrak{F}(x; a)$ is convex on $X$ with probability $1$.
\end{defn}

It is easy to see that a family of  convex functions is almost surely convex. By definition, the family $\mathfrak{F}(x; a,b,c)$ is not almost surely convex. If a family is almost surely convex, almost all of its members are convex.

A family $\mathfrak{F}(x; a)$ is said to be \emph{p-concave} if the family $-\mathfrak{F}(x; a) = \{-f(x;a): f(x;a) \in \mathfrak{F}(x; a)\}$ is $p$-convex. One can easily realize that if $\mathfrak{F}(x; a)$ is  \emph{p-concave}, then $-\mathfrak{F}(x; a)$ is  \emph{p-convex} and vice versa.

The concept of probable convexity applies equally to the cases of only one function. A function $f(x)$ is said to be $p$-convex in $X$ if it is convex in $X$ with probability at least $p$. Similarly, function $f(x)$ is said to be $p$-concave in $X$ if it is concave in $X$ with probability at least $p$.

Convex optimization refers to minimizing a convex function over a convex domain. It is also refers to maximizing a concave function over a convex domain. It has a long history and has a rich foundation. Convex problems are often considered as being easy since there exist various fast algorithms. The book by \citet{Boyd2004convex} provides an excellent introduction to the field. 

\section{Concavity of the logistic-normal function} \label{LN-concavity}

We first consider probable convexity of the following function which is called \emph{logistic-normal}:
\begin{eqnarray}
\label{ch5-eq:LN}
LN(\mbf{x}; \mbf{\mu}, \mbf{\Sigma}) =  -\frac{1}{2} (\log \tilde{\mbf{x}} - \mbf{\mu})^t \mbf{\Sigma}^{-1} (\log \tilde{\mbf{x}}  - \mbf{\mu}) - \sum_{k=1}^K \log x_k,
\end{eqnarray}
where $\mbf{\mu} \in \mathbb{R}^{K-1}, \mbf{\Sigma} \in \mathbb{S}^{K-1}_+$; $\mbf{x} \in \overline{\Delta}_K$ such that  $\log \tilde{\mbf{x}} \sim \mathcal{N}(\mbf{\mu}, \mbf{\Sigma})$. This function naturally originates from the logistic-normal distribution \citep{Aitchison1980logistic}, whose density is $p(\mbf{x}; \mbf{\mu}, \mbf{\Sigma}) \propto \exp(LN(\mbf{x}; \mbf{\mu}, \mbf{\Sigma}))$. Due to the broad use of this distribution in probabilistic modeling, the logistic-normal function plays an important role in many contexts. Nonetheless, the function itself is neither convex nor concave in $\overline{\Delta}_K$. This is one of the main reasons for why posterior estimation in nonconjugate models is often intractable.

By a thorough analysis of this function, we found the following property.

\begin{thm} \label{LN-thm:3.1}
Denote $p = 1- e^{2\log(K-1) - 0.5{(\lambda -1)^2}/{\sigma} }$ for $\lambda = \lambda_{K-1}(\mbf{\Sigma}^{-1})$ and $ \sigma = \max_i \Sigma^{-1}_{ii}$. 
Function $LN(\mbf{x};  \mbf{\mu}, \mbf{\Sigma})$ is $p$-concave over $\overline{\Delta}_K$ if $\lambda \ge 1$.
\end{thm}

This theorem essentially says that $LN$ is in fact concave under some conditions. Note that the quantity ${(\lambda - 1)^2}/{\sigma}$ is not always small. Indeed, letting $\lambda_k(\mbf{\Sigma}^{-1})$ be the $k$th eigenvalue of $\mbf{\Sigma}^{-1}$, we have $\Tr(\mbf{\Sigma}^{-1}) = \sum_{k=1}^{K-1} \lambda_k(\mbf{\Sigma}^{-1}) = \sum_{k=1}^{K-1} \Sigma^{-1}_{kk}$. When the condition number of $\mbf{\Sigma}^{-1}$ is not large, $\lambda_{K-1}(\mbf{\Sigma}^{-1})$ and $\sigma$ may be of the same order. This observation suggests that the probability bound obtained in Theorem \ref{LN-thm:3.1} is significant.

\begin{corollary} \label{LN-corol:3.1}
With notations as in Theorem \ref{LN-thm:3.1},  function $LN(\mbf{x};  \mbf{\mu}, \mbf{\Sigma})$ is almost surely concave  as $\lambda^2/{\sigma} \rightarrow +\infty$.
\end{corollary}

In the case that the least eigenvalue $\lambda$ is much larger than $\log (K-1)$, function $LN$ is concave with high probability. More concretely, if $\lambda^2 = \omega(\sigma \log K)$, i.e., $\lambda^2 /{\sigma \log K} \rightarrow +\infty$ as ${K \rightarrow +\infty}$, then   $\exp\left\{2\log(K-1) - 0.5{(\lambda -1)^2}/{\sigma}\right\}$ goes to 0. Hence the following result holds.

\begin{corollary} \label{LN-corol:3.2}
With notations as in Theorem \ref{LN-thm:3.1}, assume  that $\lambda^2 = \omega(\sigma \log K)$. Function $LN(\mbf{x};  \mbf{\mu}, \mbf{\Sigma})$ is almost surely concave  as $K \rightarrow +\infty$.
\end{corollary}

\subsection{Proof of Theorem \ref{LN-thm:3.1}} 

We will show  probable concavity of $LN$ by investigating concavity in common sense. Note that the domain $\overline{\Delta}_K$ is convex, and function $LN$ is twice differentiable over $\overline{\Delta}_K$. Hence, to see concavity, it suffices to show that the second derivative is negative semidefinite \citep{Boyd2004convex}.

Let $\Sigma_i^{-1}$ be the $i^{th}$ row of $\mbf{\Sigma}^{-1}$. The first and second partial derivatives of the function w.r.t the variables are:
\begin{eqnarray*}
\frac{\partial LN}{\partial x_i} &=&
 \left\{ \begin{array}{ll}
- \frac{1}{x_i} \Sigma_i^{-1} (\log \tilde{\mbf{x}} - \mbf{\mu}) - \frac{1}{x_i}, & i < K \\
  \frac{1}{x_K} \sum_{h=1}^{K-1} \Sigma_h^{-1} (\log \tilde{\mbf{x}} - \mbf{\mu}) - \frac{1}{x_K}, & i = K 
 \end{array}
 \right. \\
\frac{\partial^2 LN}{\partial x_i \partial x_j} &=&
 \left\{ \begin{array}{ll}
 - \frac{\Sigma^{-1}_{ij}}{x_i x_j}, & i <K, i \ne j, j < K \\
   \frac{1}{x^2_i} \Sigma_i^{-1} (\log \tilde{\mbf{x}} - \mbf{\mu}) - \frac{\Sigma^{-1}_{ii}}{x^2_i} + \frac{1}{x^2_i}, & i < K, i = j \\
   \frac{1}{x_i x_K} \sum_{h=1}^{K-1} \Sigma_{ih}^{-1}, & i < K, j = K \\
   \frac{1}{x_j x_K} \sum_{h=1}^{K-1} \Sigma_{hj}^{-1}, & i = K, j < K \\   
 - \frac{1}{x^2_K} \sum_{h=1}^{K-1} \Sigma_h^{-1} (\log \tilde{\mbf{x}} - \mbf{\mu}) - \frac{1}{x^2_K} \sum_{h=1}^{K-1} \sum_{t=1}^{K-1} \Sigma_{ht}^{-1} + \frac{1}{x^2_K}, & i = j = K. 
 \end{array}
 \right.
\end{eqnarray*}
Denote $\mbf{S} = \left(\begin{array}{ll} \mbf{\Sigma}^{-1} & \mbf{s}^t_K \\ \mbf{s}_K & s_{KK} \end{array}  \right); \mbf{U} = \left(\begin{array}{r} \mbf{\Sigma}^{-1} \\ \mbf{s}_K \end{array} \right)$, where $\mbf{s}_K = -\sum_{t=1}^{K-1} \Sigma_{t}^{-1}$ is the sum of the rows of $\mbf{\Sigma}^{-1}$, and $s_{KK}$ is the sum of all elements of $\mbf{\Sigma}^{-1}$. We can express the second derivative of $LN$ as 

\begin{eqnarray}
\nonumber 
LN'' &=& diag \frac{1}{\mbf{x}}. diag[\mbf{U}(\log \tilde{\mbf{x}} - \mbf{\mu})] .diag \frac{1}{\mbf{x}} - diag \frac{1}{\mbf{x}} . \mbf{S}. diag \frac{1}{\mbf{x}} + diag \frac{1}{\mbf{x}} . diag \frac{1}{\mbf{x}} \\  
\label{ch5-eq:Df}
 &=& diag \frac{1}{\mbf{x}}. \left(\mbf{I}_{K} -\mbf{S} + diag[\mbf{U}(\log \tilde{\mbf{x}} - \mbf{\mu})] \right). diag \frac{1}{\mbf{x}}.
\end{eqnarray}

A classical result in Algebra \citep[exercise 8.28]{AbadirM2005} says that for any symmetric $\mbf{A}$ and nonsingular $\mbf{Y}$, the product $\mbf{Y} \mbf{A} \mbf{Y}^t$ is positive semidefinite if and only if $\mbf{A}$ is positive semidefinite. Consequently, the matrix  $\mbf{I}_{K} -\mbf{S} + diag[\mbf{U}(\log \tilde{\mbf{x}} - \mbf{\mu})]$ decides negative semidefiniteness of $LN''$.

\begin{lem} \label{LN-lem:Df}
Denote $\mbf{z} = \mbf{\Sigma}^{-1}(\log \tilde{\mbf{x}} -\mbf{\mu})$. $LN''$ is negative semidefinite if  $z_1 + \cdots + z_{K-1} \ge 1$ and $ \mbf{I}_{K-1} -\mbf{\Sigma}^{-1}  + diag(\mbf{z}) \le 0$.
\end{lem}
\begin{proof}
 As discussed before, matrix $\mbf{I}_{K} -\mbf{S} + diag[\mbf{U}(\log \tilde{\mbf{x}} - \mbf{\mu})]$ decides negative semidefiniteness of $LN''$. Letting $z_K = -z_1 - \cdots - z_{K-1}$ and $\mbf{1} = (1, ..., 1)^t \in \mathbb{R}^{K-1}$, we have
 \begin{eqnarray}
 \nonumber
   \mbf{A} &=& \mbf{I}_{K} -\mbf{S} + diag[\mbf{U}(\log \tilde{\mbf{x}} - \mbf{\mu})] \\
\nonumber
 &=& \mbf{I}_K -\left(\mbf{I}_{K-1} \;\; \mbf{1}\right)^t \mbf{\Sigma}^{-1} \left(\mbf{I}_{K-1} \;\; \mbf{1}\right) + diag(z_1, ..., z_K) \\
\label{LN-eq03}
 &=& \left(\mbf{I}_{K-1} \;\; \mbf{1}\right)^t \left[\mbf{I}_{K-1} - \mbf{\Sigma}^{-1} + diag(\mbf{z})\right] \left(\mbf{I}_{K-1} \;\; \mbf{1}\right) +
 \left(\begin{array}{cc}
 \mbf{0}  & -(\mbf{z}+\mbf{1}) \\
 -(\mbf{z}+\mbf{1})^t & z_K +1
 \end{array}\right)
 \end{eqnarray}
 
 Consider the last term $\mbf{C} = \left(\begin{array}{cc}
  \mbf{0}  & -(\mbf{z}+\mbf{1}) \\
  -(\mbf{z}+\mbf{1})^t & z_K +1
  \end{array}\right)$. This matrix is of size $K \times K$, but has rank 2. It is not hard to see that all principle minors of $\mbf{C}$ are 0, except the ones which associate with the last two rows and columns. Those principle minors are $z_K +1$ and $\left| \begin{array}{cc}    0  & -z_i -1 \\  -z_i -1 & z_K +1  \end{array} \right| = z_K +1 - (z_i +1)^2$ for $i \in \{1, ..., K-1\}$. According to a classical result in Algebra \citep[exercise 8.32]{AbadirM2005}, $\mbf{C} \le 0$ if and only if  all of its principle minors are non-positive. Therefore $\mbf{C} \le 0$ if and only if $z_K +1 \le 0$.

If $\mbf{C}$ and  $\mbf{I}_{K-1} -\mbf{\Sigma}^{-1}  + diag(\mbf{z})$ are negative semidefinite, so are $\mbf{A}$ and $LN''$. This suggests that if $z_K +1 \le 0$ and $\mbf{I}_{K-1} -\mbf{\Sigma}^{-1}  + diag(\mbf{z}) \le 0$, then $LN'' \le 0$ which completes the proof.
\end{proof}

 Next we want to see under what conditions, matrix $\mbf{I}_{K-1} -\mbf{\Sigma}^{-1}  + diag(\mbf{z}) \le 0$ with the constraint of $z_1 + \cdots + z_{K-1} \ge 1$. The following theorem reveals a property whose detailed proof is presented in section \ref{proofs-random-matrix}.

\begin{thm} \label{ch5-thm:4.5}
Let $\mbf{z}$ be a Gaussian random variable with mean $0$ and covariance matrix $\mbf{A} \in \mathbb{S}^{K-1}_+$, and $\sigma = \max_i A_{ii}$. For a fixed $\mbf{S} \in \mathbb{S}^{K-1}_+$, consider $\mbf{B}= \mbf{I}_{K-1} -\mbf{S} + diag(\mbf{z})$.  Assuming $\lambda_{K-1}(\mbf{S}) \ge 1$, we have 
\[\Pr (\lambda_1(\mbf{B}) \ge 0 | z_1 + \cdots + z_{K-1} \ge 1) \le \exp\left\{ 2\log(K-1) -0.5{\left(1 -\lambda_{K-1}(\mbf{S})\right)^2}/{\sigma} \right\}.\]
\end{thm}

This theorem essentially says that under certain assumption, matrix $\mbf{B}$ is negative semidefinite with probability at least $1 -\exp\left\{ 2\log(K-1) -0.5{\left(1 -\lambda_{K-1}(\mbf{S})\right)^2}/{\sigma} \right\}$. Hence we have enough tools to prove Theorem~\ref{LN-thm:3.1}.

\begin{proof}[Proof of Theorem \ref{LN-thm:3.1}]
Consider the logistic-normal function $LN(\mbf{x}; \mbf{\mu}, \mbf{\Sigma})$, and denote $\lambda = \lambda_{K-1}(\mbf{\Sigma}^{-1})$ and $ \sigma = \max_i \Sigma^{-1}_{ii}$. As discussed before, concavity of this function over $\overline{\Delta}_K$ is decided by its second partial derivative $LN''$. Lemma~\ref{LN-lem:Df} suggests that $LN(\mbf{x}; \mbf{\mu}, \mbf{\Sigma})$ is concave if $z_1 + \cdots + z_{K-1} \ge 1$ and $\mbf{I}_{K-1} -\mbf{\Sigma}^{-1}  + diag(\mbf{z}) \le 0$, where $\mbf{z} = \mbf{\Sigma}^{-1}(\log \tilde{\mbf{x}} -\mbf{\mu})$. Note that $\mathbb{E} \mbf{z} =0$ and $cov(z) = \mbf{\Sigma}^{-1}$ since $\mathbb{E} \log \tilde{\mbf{x}} = \mbf{\mu}$ and $cov(\log \tilde{\mbf{x}}) = \mbf{\Sigma}$. Theorem~\ref{ch5-thm:4.5} implies that with the constraint of $z_1 + \cdots + z_{K-1} \ge 1$,  $\mbf{I}_{K-1} -\mbf{\Sigma}^{-1}  + diag(\mbf{z}) \le 0$ holds with probability at least $1 -\exp\left\{ 2\log(K-1) -0.5{\left(1 -\lambda\right)^2}/{\sigma} \right\}$ if $\lambda \ge 1$. This means assuming $\lambda \ge 1$, function $LN(\mbf{x}; \mbf{\mu}, \mbf{\Sigma})$ is concave with probability at least $1 -\exp\left\{ 2\log(K-1) -0.5{\left(1 -\lambda\right)^2}/{\sigma} \right\}$.
\end{proof}

\subsection{Proof of Theorem \ref{ch5-thm:4.5}} \label{proofs-random-matrix}
To prove this theorem we need some basic results from matrix algebra and the theory of random matrices.

A matrix $\mbf{A}$ is \textit{positive semidefinite} if and only if the least eigenvalue $\lambda_{\min}(\mbf{A})$ is nonnegative. If $\mbf{A}$ has $K$ eigenvalues, its trace satisfies $\Tr \mbf{A} =  \sum_{i=1}^K \lambda_i (\mbf{A})$. If $\mbf{A}$ is a random matrix, we have trace-expectation relation $\Tr \mathbb{E}\mbf{A} = \mathbb{E}(\Tr \mbf{A})$.

Consider a function $f: \mathbb{R} \rightarrow \mathbb{R}$. We define a map on a diagonal matrix $\mbf{A} \in \mathbb{S}^K$ as $f(\mbf{A}) = diag(f(A_{11}),..., f(A_{KK}))$. Similarly, a function of a symmetric matrix $\mbf{A}$ is defined by using the eigenvalue decomposition:
\[f(\mbf{A}) = \mbf{Q}.f(\mbf{\Lambda}).\mbf{Q}^t, \text{ where }  \mbf{A} = \mbf{Q}.\mbf{\Lambda}.\mbf{Q}^t \text{ and } \mbf{\Lambda} \text{ is a diagonal matrix.} \] 

The \textbf{\emph{spectral mapping theorem}} states that each eigenvalue of $f(\mbf{A})$ is equal to $f(\lambda)$ for some eigenvalue $\lambda$ of $\mbf{A}$. If $f$ is nondecreasing, then $\lambda_k(f(\mbf{A})) = f(\lambda_k(\mbf{A}))$ for any $k$ whenever $\lambda_k(\mbf{A})$ exists.

We will work with \textit{matrix exponential} which is defined for an $\mbf{A} \in \mathbb{S}^K$ by 
\[e^{\mbf{A}} = \sum_{i=0}^{\infty} \frac{\mbf{A}^i}{i!}.\] 
Note that $\lambda_k(e^{\mbf{A}}) = e^{\lambda_k(\mbf{A})}$ for any $k$ provided that $\lambda_k(\mbf{A})$ exists. The logarithm of a matrix $\mbf{A} \in \mathbb{S}_+^K$ is a matrix, denoted by $\log \mbf{A}$, such that $e^{\log \mbf{A}} = \mbf{A}$.

\begin{thm}[Golden-Thompson inequality] \label{ch2-thm:4.1}
For $\mbf{A}, \mbf{B} \in \mathbb{S}^K$, we have 
\[\Tr e^{\mbf{A}+\mbf{B}} \le \Tr(e^{\mbf{A}}.e^{\mbf{B}}).\]
\end{thm}

This is a standard result and can be found in \citep{WigdersonX2008derandomizing,{Tropp2012}}. Note that $e^{\mbf{A}}$ and $e^{\mbf{B}}$ are positive definite which implies $\Tr(e^{\mbf{A}}.e^{\mbf{B}}) \le \Tr e^{\mbf{A}}. \Tr e^{\mbf{B}}$, since according to \citet{YangF2002},  $\Tr(\mbf{A.B}) \le \Tr \mbf{A}. \Tr \mbf{B}$ if  $\mbf{A}, \mbf{B} \in \mathbb{S}_+^K$. Hence we have the following.

\begin{corollary} \label{ch2-corol:4.1}
For $\mbf{A}, \mbf{B} \in \mathbb{S}^K$, we have 
$\Tr e^{\mbf{A}+\mbf{B}} \le \Tr e^{\mbf{A}}. \Tr e^{\mbf{B}}.$
\end{corollary}

The next theorem was shown by \citet{Tropp2012}.

\begin{thm}[Laplace transform method] \label{ch2-thm:4.2}
Let $\mbf{B}$ be a random matrix of $\mathbb{S}^K$. For any real $t$, we have
\[\Pr(\lambda_1(\mbf{B}) \ge t) \le \inf_{a>0} \{ e^{a.t} \mathbb{E} \Tr e^{a.\mbf{B}}\}.\]
\end{thm}

\begin{lem}\label{ch2-lem:4.3}
Consider a matrix $\mbf{B} \in \mathbb{S}^K$ and a nonnegative real $a$. We have \[\mathbb{E} \Tr e^{a.\mbf{B}} \le K \mathbb{E} e^{a\lambda_1(\mbf{B})}.\]
\end{lem}
\begin{proof}
Since the trace of $\mbf{B}$ equals the sum of its eigenvalues, we have $\Tr \mbf{B}  \le K\lambda_1(\mbf{B})$. Hence $\mathbb{E} \Tr e^{a.\mbf{B}} \le K \mathbb{E} \lambda_1(e^{a \mbf{B}}) \le K \mathbb{E} e^{\lambda_1(a \mbf{B})} = K \mathbb{E} e^{a \lambda_1(\mbf{B})}$, where the last inequality is derived by using the spectral mapping theorem.
\end{proof}

% % % % % % % % % % % % % % % % % % % % % % % % %

\begin{lem}\label{ch5-lem:4.6}
Consider a Gaussian random vector $\mbf{z}$ with mean $0$ and covariance matrix $\mbf{A} \in \mathbb{S}^K_+$. Let $\sigma_i = A_{ii}$ be the $i^{th}$ diagonal entry of $\mbf{A}$, and $\sigma = \max_i \sigma_{i}$. Then for any real $a >0$, we have
$\mathbb{E} \Tr e^{a.diag(\mbf{z})}  = \sum_{k=1}^K e^{a^2 \sigma_k /2} \le K e^{a^2 \sigma /2}$.
\end{lem}

\begin{proof}
Note that
%$\Tr e^{a.diag(\mbf{z})} = \Tr \sum_{i=0}^{\infty} \frac{a^i}{i!} diag^i (\mbf{z}) = \Tr \sum_{i=0}^{\infty} \frac{a^i}{i!} diag(z^i_1, ..., z^i_K)  = \sum_{i=0}^{\infty} \frac{a^i}{i!} \Tr diag(z^i_1, ..., z^i_K)  = \sum_{i=0}^{\infty} \frac{a^i}{i!} \sum_{k=1}^K z^i_k = \sum_{k=1}^K \sum_{i=0}^{\infty} \frac{a^i}{i!} z^i_k = \sum_{k=1}^K e^{a.z_k}$.
\begin{eqnarray*}
   \Tr e^{a.diag(\mbf{z})}
  &=& \Tr \sum_{i=0}^{\infty} \frac{a^i}{i!} diag^i (\mbf{z}) \\
  &=& \Tr \sum_{i=0}^{\infty} \frac{a^i}{i!} diag(z^i_1, ..., z^i_K) \\
  &=& \sum_{i=0}^{\infty} \frac{a^i}{i!} \Tr diag(z^i_1, ..., z^i_K) \\
  &=& \sum_{i=0}^{\infty} \frac{a^i}{i!} \sum_{k=1}^K z^i_k
  = \sum_{k=1}^K \sum_{i=0}^{\infty} \frac{a^i}{i!} z^i_k
  = \sum_{k=1}^K e^{a z_k}
\end{eqnarray*}
Hence $\mathbb{E} \Tr e^{a.diag(\mbf{z})} = \mathbb{E} \sum_{k=1}^K e^{a.z_k} = \sum_{k=1}^K \mathbb{E} e^{a.z_k}$.

By assumption, $z_k$ is a Gaussian variable with mean 0 and variance $\sigma_k$. Using the generating function of Gaussian, we have $\mathbb{E} e^{a.z_k} = e^{a^2 \sigma_k /2}$. So substituting these quantities into the expectation in the last paragraph completes the proof.
\end{proof}

\begin{proof}[Proof of Theorem \ref{ch5-thm:4.5}]

We have
\begin{eqnarray*}
\nonumber
   \Pr (\lambda_1(\mbf{B}) \ge 0 | z_1 + \cdots + z_{K-1} \ge 1) &\le&
   \Pr( \lambda_1(\mbf{B}) \ge 0 ) \\
   &\le& \inf_{a>0} \left\{ \mathbb{E} \Tr e^{a \mbf{B}} \right\} \\
      & & \text{(Laplace transform method)} \\
      &=& \inf_{a>0} \left\{ \mathbb{E} \Tr e^{a [\mbf{I}_{K-1} -\mbf{S} + diag(\mbf{z})]} \right\} \\
      &\le& \inf_{a>0} \left\{ \mathbb{E} \left( \Tr e^{a [\mbf{I}_{K-1} -\mbf{S}]}. \Tr e^{a.diag(\mbf{z})} \right) \right\} \\
      & & \text{(Corollary \ref{ch2-corol:4.1})} \\
      &=& \inf_{a>0} \left\{  \Tr e^{a [\mbf{I}_{K-1} -\mbf{S}]}. \mathbb{E} \Tr e^{a.diag(\mbf{z})} \right\} \\
      &\le& \inf_{a>0} \left\{  \Tr e^{a [\mbf{I}_{K-1} -\mbf{S}]}. (K-1). e^{a^2 \sigma/2} \right\} \\
      & & \text{(Lemma \ref{ch5-lem:4.6})} \\
      &=& \inf_{a>0} \left\{ (K-1). e^{a^2 \sigma/2}.  \Tr e^{a [\mbf{I}_{K-1} -\mbf{S}]} \right\} \\
      &\le& \inf_{a>0} \left\{ (K-1). e^{ a^2 \sigma/2}.  (K-1). \lambda_1 (e^{a [\mbf{I}_{K-1} -\mbf{S}]}) \right\} 
\end{eqnarray*}   
\begin{eqnarray*}
\Pr (\lambda_1(\mbf{B}) \ge 0 | z_1 + \cdots + z_{K-1} \ge 1) 
      &\le& \inf_{a>0} \left\{ (K-1)^2. e^{a^2 \sigma/2}. e^{\lambda_1 (a [\mbf{I}_{K-1} -\mbf{S}])} \right\} \\
      & & \text{(Spectral mapping theorem)} \\
      &=& \inf_{a>0} \left\{ (K-1)^2. e^{a^2 \sigma/2}. e^{a - a\lambda_{K-1} (\mbf{S})} \right\} \\
      &=& \inf_{a>0} \left\{ (K-1)^2. e^{a^2 \sigma/2 +a -a\lambda_{K-1} (\mbf{S})} \right\} \\
      &=& (K-1)^2 \exp\left\{ -\frac{\left(1 -\lambda_{K-1}(\mbf{S})\right)^2}{2\sigma} \right\}.
   \end{eqnarray*}
Note that the last equality is obtained by minimizing the function $a^2 \frac{\sigma}{2} +a -a\lambda_{K-1} (\mbf{S})$ for $a > 0$ conditioned on $1 \le \lambda_{K-1}(\mbf{S})$.
\end{proof}

\section{MAP inference of topic mixtures in CTM} \label{ch5-sec:MAP-in-CTM}

We next study convexity of a family originated from the topic modeling literature. In particular, we are interested in the problem of estimating topic mixtures (posterior distributions) in correlated topic models (CTM) \citep{BleiL07}. This problem is intractable by traditional approaches \citep{BleiL07,AhmedX2007}. We will show that in fact this problem is tractable under some conditions, by showing probable concavity.

The correlated topic model assumes that a corpus is composed from $K$ topics $\mbf{\beta}_1, ..., \mbf{\beta}_K$, and a document $\mbf{d}$ arises from the following generative process:
\begin{enumerate}
  \item Draw $\mbf{x} | \mbf{\mu}, \mbf{\Sigma} \sim \mathcal{N}(\mbf{\mu}, \mbf{\Sigma})$
  \item For the $n^{th}$ word of $\mbf{d}$:
  \begin{itemize}
    \item[-] draw topic assignment $z_{dn} | \mbf{x} \sim \mathcal{M}(f(\mbf{x}))$
    \item[-] draw word $w_{dn}| z_{dn}, \mbf{\beta} \sim \mathcal{M}(\mbf{\beta}_{z_{dn}})$.
  \end{itemize}
\end{enumerate}
where $\mathcal{N}(\mbf{\mu}, \mbf{\Sigma})$ is the normal distribution with mean $\mbf{\mu}$ and covariance $\mbf{\Sigma}$; $\mathcal{M}(\cdot)$ is the multinomial distribution; $f(\mbf{x})$ maps a natural parameterization of the topic proportion to the mean parameterization:
\begin{equation}\label{ch5-eq:001}
\mbf{\theta} = f(\mbf{x}) = \frac{e^{\mbf{x}}}{\sum_{k=1}^K e^{x_k}}.
\end{equation}
This logistic transformation maps a $K$-dimensional vector $\mbf{x}$ to a $(K-1)$-dimensional vector $\mbf{\theta}$. Hence various $\mbf{x}$'s can correspond to a single $\mbf{\theta}$. Fixing $x_K =0$, the transformation (\ref{ch5-eq:001}) means that $\mbf{\theta}$ follows the logistic-normal distribution \citep{BleiL07}. According to \citet{Aitchison1980logistic}, the density function of $\mbf{\theta}$ is thus 
\begin{equation}\label{ch5-eq:001.1}
p(\mbf{\theta}; \mbf{\mu, \Sigma}) = \frac{1}{\sqrt{\det(2\pi \mbf{\Sigma})}} \exp \left(-\frac{1}{2} (\log \tilde{\mbf{\theta}} - \mbf{\mu})^t \mbf{\Sigma}^{-1} (\log \tilde{\mbf{\theta}} - \mbf{\mu}) - \sum_{k=1}^K \log \theta_k \right),
\end{equation}
where $\mbf{\mu} \in \mathbb{R}^{K-1}, \mbf{\Sigma} \in \mathbb{S}^{K-1}_+$. Note that $\mbf{\theta}$ is derived from $\mbf{x}$ by (\ref{ch5-eq:001}). Hence $\log \tilde{\mbf{\theta}}$ is a normal random variable with mean $\mbf{\mu}$ and covariance $\mbf{\Sigma}$.

One of the most interesting tasks in this model is the posterior estimation of topic mixtures for documents. More concretely, given the model parameters $\Upsilon = \{\mbf{\beta}, \mbf{\mu}, \mbf{\Sigma}\}$, we are interested in the following problem for a given document $\mbf{d}$:
\begin{eqnarray}
\nonumber
\mbf{\theta}^* &=& \arg \max_{\mbf{\theta} \in \Delta_K} \Pr(\mbf{\theta} | \mbf{d}, \Upsilon) \\
\label{ch5-eq:002}
 &=&  \arg \max_{\mbf{\theta} \in \Delta_K} \Pr(\mbf{\theta}, \mbf{d} | \Upsilon) 
\end{eqnarray}

\begin{lem} \label{ch5-lem:2.1}
Given a CTM model with parameters $\Upsilon = \{\mbf{\beta}, \mbf{\mu}, \mbf{\Sigma}\}$ and a document $\mbf{d}$, the MAP problem (\ref{ch5-eq:002}) can be reformulated as
\begin{equation}\label{ch5-eq:003}
\mbf{\theta}^* = \arg \max_{\mbf{\theta} \in \overline{\Delta}_K} \sum_j d_j \log \sum_{k=1}^K \theta_k \beta_{kj} -\frac{1}{2} (\log \tilde{\mbf{\theta}} - \mbf{\mu})^t \mbf{\Sigma}^{-1} (\log \tilde{\mbf{\theta}}  - \mbf{\mu}) - \sum_{k=1}^K \log \theta_k.
\end{equation}
\end{lem}

\begin{proof}
We have 
\[\mbf{\theta}^* = \arg \max_{\mbf{\theta} \in \Delta_K} \Pr(\mbf{\theta}, \mbf{d} | \Upsilon) = \arg \max_{\mbf{\theta} \in \Delta_K} \log \Pr(\mbf{\theta}, \mbf{d} | \Upsilon) = \arg \max_{\mbf{\theta} \in \Delta_K} \log \Pr(\mbf{d} | \mbf{\theta}, \Upsilon) + \log \Pr(\mbf{\theta} | \Upsilon).\] 
Note that $\log \Pr(\mbf{d} | \mbf{\theta}, \Upsilon) = \sum_j d_j \log \sum_{k=1}^K \theta_k \beta_{kj}$, and the density of the logistic-normal distribution is given in (\ref{ch5-eq:001.1}). Hence 
\[\mbf{\theta}^* = \arg \max_{\mbf{\theta} \in \Delta_K}  \sum_j d_j \log \sum_{k=1}^K \theta_k \beta_{kj} -\frac{1}{2} (\log \tilde{\mbf{\theta}} - \mbf{\mu})^t \mbf{\Sigma}^{-1} (\log \tilde{\mbf{\theta}}  - \mbf{\mu}) - \sum_{k=1}^K \log \theta_k - \frac{1}{2} \log \det(2\pi \mbf{\Sigma}).\] 
Since any point on the boundary of $\Delta_K$ makes the objective function undefined and hence is not optimal. Therefore, ignoring the boundary of $\Delta_K$ and the constant in the objective function completes the proof.
\end{proof}

Loosely speaking, Lemma \ref{ch5-lem:2.1} says that posterior estimation of topic mixtures in CTM is in fact an optimization problem. The objective function is well-defined on $\overline{\Delta}_K$. It is worth remarking that this function is neither concave nor convex in general. Hence maximizing it over $\overline{\Delta}_K$  is intractable in the worse case.

\subsection{Some results}

Let the model parameters $\Upsilon = \{\mbf{\beta}, \mbf{\mu}, \mbf{\Sigma}\}$ be fixed, where $\mbf{\beta}_k \in \Delta_V, \mbf{\mu} \in \mathbb{R}^{K-1}, \mbf{\Sigma} \in \mathbb{S}^{K-1}_+$. Consider the following family, parameterized by $\mbf{d}$:
\begin{eqnarray}
\label{ch5-eq:004}
CTM(\mbf{\theta}; \mbf{d}, \Upsilon) = \{ f(\mbf{\theta}; \mbf{d}, \Upsilon): 
\mbf{\theta} \in \overline{\Delta}_K, \log \tilde{\mbf{\theta}} \sim \mathcal{N}( \mbf{\mu}, \mbf{\Sigma}) \}.
\end{eqnarray}
where $f(\mbf{\theta}; \mbf{d}, \Upsilon) = \sum_j d_j \log \sum_{k=1}^K \theta_k \beta_{kj} -\frac{1}{2} (\log \tilde{\mbf{\theta}} - \mbf{\mu})^t \mbf{\Sigma}^{-1} (\log \tilde{\mbf{\theta}}  - \mbf{\mu}) - \sum_{k=1}^K \log \theta_k$. This family contains all possible instances of the  problem (\ref{ch5-eq:003}). Hence, analyzing this family means analyzing the problem of estimating topic mixtures in CTM.

Consider a  member $f(\mbf{\theta}; \mbf{d}, \Upsilon)$. Note that  $\mbf{d}$ and $\mbf{\beta}$ are always nonnegative in practices of topic modeling. Hence the first term in $f(\mbf{\theta}; \mbf{d}, \Upsilon)$ is always concave over $\overline{\Delta}_K$. It implies that concavity of $f(\mbf{\theta}; \mbf{d}, \Upsilon)$ is heavily determined by the logistic-normal term $y = -\frac{1}{2} (\log \tilde{\mbf{\theta}} - \mbf{\mu})^t \mbf{\Sigma}^{-1} (\log \tilde{\mbf{\theta}}  - \mbf{\mu}) - \sum_{k=1}^K \log \theta_k$. If this term is concave, then $f(\mbf{\theta}; \mbf{d}, \Upsilon)$ is concave. Combining these observations with Theorem \ref{LN-thm:3.1}, Corollary \ref{LN-corol:3.1}, and Corollary \ref{LN-corol:3.2}, we arrive at the following results for CTM.

%\begin{thm} \label{ch5-thm:2.1}
%For fixed $\Upsilon$, suppose $\lambda_{K-1}(\mbf{\Sigma}^{-1}) \ge 1$. Family $CTM(\mbf{\theta}; \mbf{d}, \Upsilon)$ is probably concave  with probability at least $1- \exp\left\{-0.5{(\lambda_{K-1}(\mbf{\Sigma}^{-1}) - 1)^2}/{\sigma} + 2\log (K-1)\right\}$, where $\sigma$ is the largest amongst the diagonal elements of $\mbf{\Sigma}^{-1}$.
%\end{thm}

\begin{thm} \label{ch5-thm:2.1}
Let $\Upsilon$ be fixed, $\sigma = \max_i \Sigma_{ii}^{-1}, \lambda = \lambda_{K-1}(\mbf{\Sigma}^{-1})$, and $p = 1- e^{2\log (K-1) -0.5{(\lambda - 1)^2} / \sigma}$.  Assuming $\lambda \ge 1$, family $CTM(\mbf{\theta}; \mbf{d}, \Upsilon)$ is $p$-concave over $\overline{\Delta}_K$.
\end{thm}

\begin{corollary} \label{ch5-corol:2.1}
With notations as in Theorem \ref{ch5-thm:2.1}, family $CTM(\mbf{\theta}; \mbf{d}, \Upsilon)$ is almost surely concave  as $\lambda^2 /{\sigma} \rightarrow +\infty$.
\end{corollary}

\begin{corollary} \label{ch5-corol:2.2}
With notations as in Theorem \ref{ch5-thm:2.1}, assume  that $\lambda^2 = \omega(\sigma \log K)$. Family $CTM(\mbf{\theta}; \mbf{d}, \Upsilon)$ is almost surely concave  as $K \rightarrow +\infty$.
\end{corollary}

\subsection{Implication to related models}

Many nonconjugate models employ the Gaussian distribution to model correlation of hidden topics, including those by \citet{BleiL2006DTM,PutthividhyaAN09,PutthividhyaAN10,SalomatinYL09,{Miao+2012LAA}}. The analysis for CTM is very general for the case of logistic-normal priors. Therefore, the results for CTM can be easily derived for  other nonconjugate topic models. Here we take DTM \citep{BleiL2006DTM} and IFTM \citep{PutthividhyaAN09} into consideration as two specific examples.

The \emph{Independent Factor Topic Model} (IFTM) by \citet{PutthividhyaAN09} is a variant of CTM in which $\mbf{\mu}$ is replaced with $\mbf{\mu}' = \mbf{As} + \mbf{\mu}$ to model independent sources that compose correlated topics. A slight modification to our analysis would yield interesting results for the corresponding family, denoting $\Upsilon' = \{\mbf{\beta}, \mbf{\mu}', \mbf{\Sigma}\}$, 
\[
IFTM(\mbf{\theta}; \mbf{d}, \Upsilon')=
\{ f(\mbf{\theta}; \mbf{d}, \Upsilon'):
\mbf{\theta} \in \overline{\Delta}_K, \log \tilde{\mbf{\theta}} \sim \mathcal{N}(\mbf{\mu}', \mbf{\Sigma}) \}.\]

\begin{thm}  \label{ch5-thm:2.4}
Let $\Upsilon'$ be fixed, $\sigma = \max_i \Sigma_{ii}^{-1}, \lambda = \lambda_{K-1}(\mbf{\Sigma}^{-1})$, and $p = 1- e^{2\log (K-1) -0.5{(\lambda - 1)^2} / \sigma}$.  Assuming $\lambda \ge 1$, family $IFTM(\mbf{\theta}; \mbf{d}, \Upsilon')$ is $p$-concave over $\overline{\Delta}_K$.
\end{thm}

The \emph{Dynamic Topic Model} (DTM) by \citet{BleiL2006DTM} also employs Gaussian priors to model correlation. Those priors are separable, i.e., having diagonal covariance matrices. Let
$DTM(\mbf{\theta}; \mbf{d}, \mbf{\beta}, \mbf{\alpha}, \sigma)$ be defined similarly with (\ref{ch5-eq:004}), where $\mbf{\Sigma}^{-1} = diag(\sigma,...,\sigma)$. For this family, note that $\lambda_{K-1}(\mbf{\Sigma}^{-1}) = \sigma$. Hence, Theorem \ref{ch5-thm:2.1} implies

\begin{thm} \label{ch5-thm:2.5}
For fixed $\{\mbf{\beta}, \mbf{\alpha}, \sigma\}$, if $\sigma \ge 1$ then  family $DTM(\mbf{\theta}; \mbf{d}, \mbf{\beta}, \mbf{\alpha}, \sigma)$ is probably concave with probability at least $1- e^{2\log (K-1) -0.5{\sigma}-0.5/\sigma +1}$.
\end{thm}

\section{A fast algorithm for learning CTM} \label{ch5-sec:learning-CTM}

In this section we discuss an application of the findings in Section \ref{ch5-sec:MAP-in-CTM} to designing an efficient algorithm for learning CTM. Nonconjugacy of the prior over $\mbf{\theta}$ poses various drawbacks and precludes using sampling techniques. Hence \citet{BleiL07} proposed to use variational Bayesian methods to approximate the posterior distributions of latent variables. Variational Bayesian methods have been employed heavily for learning many other nonconjugate models \citep{SalomatinYL09,PutthividhyaAN10,PutthividhyaAN09,BleiL2006DTM,Miao+2012LAA}. The use of  simplified distributions to approximate the true posterior often results in more parameters to be optimized when learning a model. (For example, the method by \citet{BleiL07} maintains $K$ Gaussian distributions for each document.) Hence it could be problematic when the corpus is large. 

Learning CTM and other related models can be made significantly simpler by using our analysis. Indeed, to estimate the posterior $(P(\mbf{\theta} | \mbf{d}, \Upsilon))$ of topic mixtures, one can exploit fast algorithms for convex optimization. The analysis in Section \ref{ch5-sec:MAP-in-CTM} provides a theoretically reasonable justification for such an exploitation.  Once $\mbf{\theta}$ had been inferred for each document in the training data, one can follow the approach by \cite{ThanH2012fstm} to estimate topics $\mbf{\beta}$. A Gaussian prior is also easily estimated when all $\mbf{\theta}$ of the training documents are known.

\subsection{Derivation of the algorithm}

Our proposed algorithm for learning CTM is presented in Algorithm \ref{ch5-alg:learning-CTM} which is an alternative algorithm similar to EM. This algorithm tries to maximize the following regularized joint likelihood of the training corpus $\mathcal{C}$:
\begin{eqnarray*}
    L(\mbf{\beta}, \mbf{\mu}, \mbf{\Sigma}) &=& \sum_{\mbf{d} \in \mathcal{C}} \log \Pr(\mbf{\theta}, \mbf{d} | \mbf{\beta}, \mbf{\mu}, \mbf{\Sigma})  - \frac{M}{2} \alpha \Tr \mbf{\Sigma}^{-1} \\
    &=& \sum_{\mbf{d} \in \mathcal{C}} \sum_j d_j \log \sum_{k=1}^K \theta_k \beta_{kj} - \frac{1}{2}\sum_{\mbf{d} \in \mathcal{C}} (\log \tilde{\mbf{\theta}} - \mbf{\mu})^t \mbf{\Sigma}^{-1} (\log \tilde{\mbf{\theta}} - \mbf{\mu}) \\
    && -\frac{M}{2} \log \det \mbf{\Sigma}  - \frac{M}{2} \alpha \Tr \mbf{\Sigma}^{-1} + constant.
\end{eqnarray*}

The main reason for imposing a regularization term $\alpha \Tr \mbf{\Sigma}^{-1}$ on the joint likelihood is to control the eigenvalues of the learned $\mbf{\Sigma}^{-1}$. Large $\alpha$ often prevents the eigenvalues of $\mbf{\Sigma}^{-1}$ from increasing. On the other hand, small values of $\alpha$ play the role as allowing large eigenvalues of $\mbf{\Sigma}^{-1}$. In the latter case, Corollary \ref{ch5-corol:2.1} and Corollary \ref{ch5-corol:2.2} suggest that estimation of topic mixtures ($\mbf{\theta}$) is more likely to be a concave problem, and thus can be done efficiently.

In Step 1 which does posterior inference for each document, we use the Online Frank-Wolfe algorithm \citep{Hazan2012OFW} to maximize the joint probability $\Pr(\mbf{\theta}, \mbf{d} | \mbf{\beta}, \mbf{\mu}, \mbf{\Sigma})$.  This algorithm theoretically converges to the optimal solutions, provided that the optimization problem is concave.\footnote{In practice we can approximate $\overline{\Delta}_K$ by $\Delta_{\epsilon} = \{\mbf{\theta}: \sum_{k=1}^K \theta_k = 1, {\theta}_i \ge \epsilon, \forall i\}$ for a very small constant $\epsilon$, says $\epsilon = 10^{-10}$. Hence the online Frank-Wolfe algorithm should be slightly modified accordingly.} Note that Algorithm \ref{alg:Online-Frank-Wolfe} is a slight but careful modification of the general algorithm by \cite{Hazan2012OFW}, and in fact is similar with the algorithm which is presented by \cite{Clarkson2010}. 

In Step 2, we fix $\mbf{\theta}_d$ which has been inferred for each document $\mbf{d} \in \mathcal{C}$ in Step 1, and maximize $L(\mbf{\beta}, \mbf{\mu}, \mbf{\Sigma})$ to estimate the model parameters. Solving for $\mbf{\beta}$ can be done independently of $\mbf{\mu}, \mbf{\Sigma}$. Hence by using the same argument as \citet{ThanH2012fstm}, we can arrive at the formula (\ref{ch5-eq:06}) for updating topics. Maximizing the term relating to $\mbf{\mu}$ in $L(\mbf{\beta}, \mbf{\mu}, \mbf{\Sigma})$ will lead to (\ref{ch5-eq:07}) for updating $\mbf{\mu}$.

Take  $\mbf{\Sigma}$ into consideration: $L_{\alpha} = - \frac{1}{2} \sum_{\mbf{d} \in \mathcal{C}} (\log \tilde{\mbf{\theta}}_d - \mbf{\mu})^t \mbf{\Sigma}^{-1} (\log \tilde{\mbf{\theta}}_d - \mbf{\mu}) -\frac{M}{2} \log \det \mbf{\Sigma} - \frac{M}{2} \alpha \Tr \mbf{\Sigma}^{-1}$. Its derivative with respect to $\mbf{\Sigma}^{-1}$ is $\nabla L_{\alpha} = - \frac{1}{2} \sum_{\mbf{d} \in \mathcal{C}} (\log \tilde{\mbf{\theta}}_d - \mbf{\mu}) (\log \tilde{\mbf{\theta}}_d - \mbf{\mu})^t + \frac{M}{2}  \mbf{\Sigma} - \frac{M}{2} \alpha \mbf{I}_{K-1}$. Solving $\nabla L_{\alpha} = 0$, one can derive (\ref{ch5-eq:08}) for updating $\mbf{\Sigma}$.

\begin{algorithm}[t]
   \caption{fCTM: a fast algorithm for learning  correlated topic models}
   \label{ch5-alg:learning-CTM}
\begin{algorithmic}
   \STATE {\bfseries Input:} a corpus $\mathcal{C} = \{\mbf{d_1}, ..., \mbf{d}_M\}$, and a positive constant $\alpha$.
   \STATE {\bfseries Output:} $\mbf{\beta}, \mbf{\mu}, \mbf{\Sigma}$.
   \STATE {Initialize} $\mbf{\beta}, \mbf{\mu}, \mbf{\Sigma}$, and then alternate the following two steps until convergence.
   \STATE {\bfseries Step 1:} for each document $\mbf{d}$, use Algorithm \ref{alg:Online-Frank-Wolfe} to solve for
	\begin{equation} \label{ch5-eq:09}
	\mbf{\theta}_d = \arg \max_{\mbf{\theta} \in \overline{\Delta}_K} \log \Pr(\mbf{\theta}, \mbf{d} | \mbf{\beta}, \mbf{\mu}, \mbf{\Sigma})
	\end{equation}
   \STATE {\bfseries Step 2:} compute
   \begin{eqnarray}
     \label{ch5-eq:06}
     \beta_{kj} &\propto& \sum_{\mbf{d} \in \mathcal{C}} d_j \theta_{dk}, \\
     \label{ch5-eq:07}
     \mbf{\mu} &=& \frac{1}{M} \sum_{\mbf{d} \in \mathcal{C}} \log \tilde{\mbf{\theta}}_d, \\
     \label{ch5-eq:08}
     \mbf{\Sigma} &=& \alpha \mbf{I}_{K-1} + \frac{1}{M} \sum_{\mbf{d} \in \mathcal{C}} (\log \tilde{\mbf{\theta}}_d - \mbf{\mu})(\log \tilde{\mbf{\theta}}_d - \mbf{\mu})^t.
   \end{eqnarray}
\end{algorithmic}
\end{algorithm}

\begin{algorithm}[tp]
   \caption{Online Frank-Wolfe (OFW)}
   \label{alg:Online-Frank-Wolfe}
\begin{algorithmic}
   \STATE {\bfseries Input: } document $\boldsymbol{d}$, and model $\Upsilon = \{\mbf{\beta}, \mbf{\mu}, \mbf{\Sigma}\}$.
   \STATE {\bfseries Output:} $\boldsymbol{\theta}$  that maximizes \\
   $\;\;\;\;\; f(\boldsymbol{\theta}) = \sum_j d_j \log \sum_{k=1}^K \theta_k \beta_{kj} - \frac{1}{2} (\log \tilde{\mbf{\theta}} - \mbf{\mu})^t \mbf{\Sigma}^{-1} (\log \tilde{\mbf{\theta}} - \mbf{\mu}) - \sum_{k=1}^{K} \log \theta_k.$
   \STATE Initialize $\mbf{\theta}_1$ arbitrarily in $\overline{\Delta}_K$.
   \FOR{ $\ell = 1, ..., \infty$}
   \STATE Pick  $f_{\ell}$ uniformly from \\ $\;\;\; \{\sum_j d_j \log \sum_{k=1}^K \theta_{k} \beta_{kj}; \;\; - \frac{1}{2} (\log \tilde{\mbf{\theta}} - \mbf{\mu})^t \mbf{\Sigma}^{-1} (\log \tilde{\mbf{\theta}} - \mbf{\mu})  - \sum_{k=1}^{K} \log \theta_{k} \}$
   \STATE $F := \frac{1}{\ell} \sum_{h=1}^{\ell} f_h$
   \STATE $i' := \arg \max_i   \nabla F(\boldsymbol{\theta}_{\ell})_{i}$; (maximal partial gradient)
   \STATE $\alpha := {2}/{(\ell +2)}$;
   \STATE $\boldsymbol{\theta}_{\ell +1} := \alpha \boldsymbol{e}_{i'} +(1-\alpha)\boldsymbol{\theta}_{\ell}$.
   \ENDFOR
   \STATE Return $\boldsymbol{\theta}_{*}$ with largest $f$ amongst $\boldsymbol{\theta}_{1}, \boldsymbol{\theta}_{2}, ...$ 
\end{algorithmic}
\end{algorithm}

\subsection{Why may Online Frank-Wolfe help?}
 We now discuss why OFW can do  inference well in CTM even though  inference is generally non-concave. In our observation, good performance of OFW originates mainly from (1) the  probable concavity of the inference problem for which many instances in practice are concave, and from (2) the stochastic nature that allows OFW to get out of local optima to get closer to global ones.
 
 Note that Step~1 of the learning algorithm has to do inference many times, each for a specific document. Hence, we have a family of inference instances. The analysis in Section~\ref{ch5-sec:MAP-in-CTM} reveals that under some conditions, inferring topic mixtures in CTM is in fact concave. In other words, there may be many concave instances in Step~1. For them, OFW is guaranteed to converge to the optimal solutions \citep{Hazan2012OFW}.
 
 For non-concave instances, inference is more difficult as there might be many local optima. Nonetheless, OFW is able find good approximate solutions due to at least two reasons. First, OFW is able to get out of local optima to reach closer to the global ones owing to its stochastic nature in selecting  directions. Such an ability is intriguing that is missing in traditional deterministic algorithms for non-concave optimization. Second, due to the greedy nature, OFW is able to get close to local optima. 
 
\subsection{Experiments}

This section is dedicated to answering the following three questions. \emph{(a) How fast does fCTM do? (b) How good are the models learned by fCTM?} By answering these questions, we will see more clearly some benefits of studying probable convexity and the use of SGDs.  \emph{(c) How well and how fast does OFW resolve the inference problem in practice?} This question arises naturally as OFW \citep{Hazan2012OFW} was originally designed for concave problems while inference in CTM is non-concave in the worse case. Answer to this question also supports  our highlight that SGDs might be a practical choice for non-concave optimization. 

Four benchmark datasets were used in our investigation: KOS with 3430 documents, NIPS with 1500 documents, Enron with 39861 documents, and Grolier with 29762 documents.\footnote{
KOS, NIPS, and Enron were retrieved from \url{http://archive.ics.uci.edu/ml/datasets/}. \\
Grolier was retrieved from \url{http://cs.nyu.edu/~roweis/data.html}} For each dataset, we used 80\% for learning models, and the remaining part was used to check the quality and efficiency of OFW.

\subsubsection{How fast does fCTM perform?}

\begin{figure}
\centering
  \includegraphics[width=\textwidth]{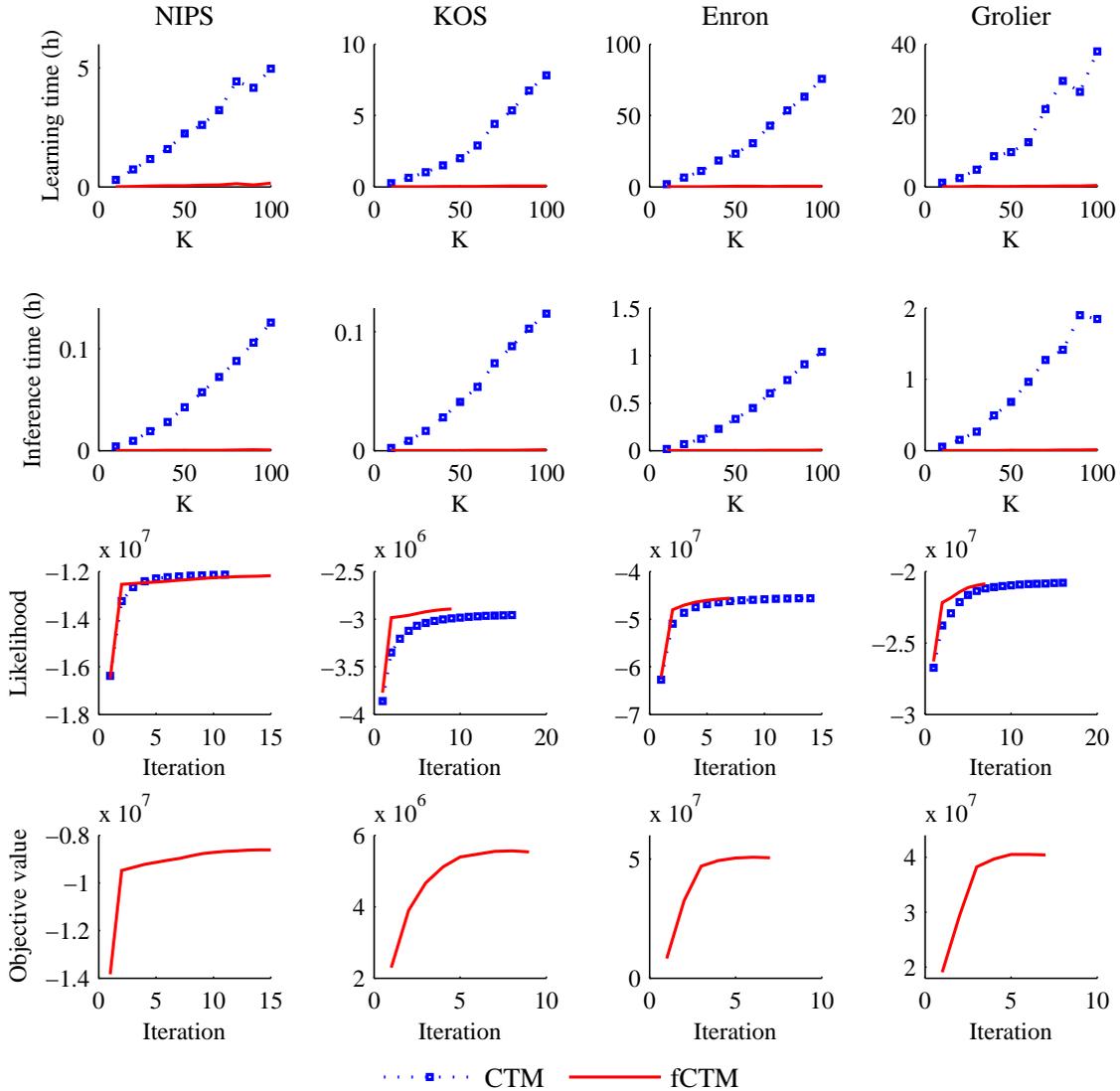}\\
  \caption[Performance of fCTM and CTM]{Performance of fCTM and CTM as the number $K$ of topics increases. Lower is better for inference/learning time, whereas higher is better for likelihood. The last two rows show how fast fCTM can reach convergence for $K=100$. We observe that fCTM often learns 60-170 times faster than CTM.}
  \label{ch5-fig:time-fCTM-CTM}
\end{figure}

To answer the first two questions and to see advantages of our algorithm (fCTM), we took the variational Bayesian method (denoted as CTM) by \citet{BleiL07} into comparison. {We used the same convergence criteria for fCTM and CTM: relative improvement of objective functions is less than $10^{-6}$ for inference of each document, and $10^{-3}$ for learning; at most 100 iterations are allowed to do inference. We used default settings for some other parameters of CTM. To avoid doing cross-validation for selecting the best value of $\alpha$ in fCTM, we used $\alpha=1$ as the default setting.}

Figure~\ref{ch5-fig:time-fCTM-CTM} records some statistics from learning and inference. We observed that fCTM learns significantly faster than CTM. Similar behavior holds when doing inference for each document. In our observations, fCTM often learns 60-170 times faster than CTM. Speedy learning of fCTM can be explained by the fact that Step~1 is done efficiently by OFW which has a linear convergence rate, provided that the inference problem is concave. In the cases of non-concave problems, OFW is still able to find efficiently approximate solutions. We observe that OFW often works 50-170 times faster than the variational method. In contrast, CTM did slowly because many auxiliary parameters need to be optimized when doing inference for each document. Furthermore, the variational method  is not guaranteed to converge quickly. Figure~\ref{ch5-fig:time-fCTM-CTM} shows that CTM often needs intensive time to do inference.

\textit{Convergence speed:}
The last two rows in Figure~\ref{ch5-fig:time-fCTM-CTM} show how fast CTM and fCTM can reach convergence. Both methods can reach convergence in a relatively few iterations. We observe that both methods rarely need 20 iterations to reach convergence; they both can reach stable after 10 iterations. Such a behavior would be beneficial when working in the cases of limited time.

\subsubsection{How good are the models learned by fCTM?}

\textit{Likelihood} and \textit{coherence} are used to see the quality of models learned from data. Coherence is used to assess quality (goodness and interpretability) of individual topics. It has been observed to  reflect well human assessment \citep{Mimno2011OptimizingCoherence}. 

To calculate the coherence of a topic $k$, we first choose the set $V^k = \{v_1^k, ..., v_t^k\}$ of the top $t$ terms that have highest probabilities in that topic, and then compute
\[
C(k, V^k) = \sum_{m=2}^{t} \sum_{l=1}^{m-1} \log \frac{D(v_m^k , v_l^k) +1}{D(v_l^k)}
\] 
where $D(v)$ is the document frequency of term $v$, $D(u, v)$ is the number of documents that contain both terms $u$ and $v$. In our experiments, we chose top $t=20$ terms for investigation, and coherence of individual topics is averaged: 
\[coherence = \frac{1}{K} \sum_{k=1}^{K} C(k, V^k).\]

\begin{figure}
\centering
  \includegraphics[width=\textwidth]{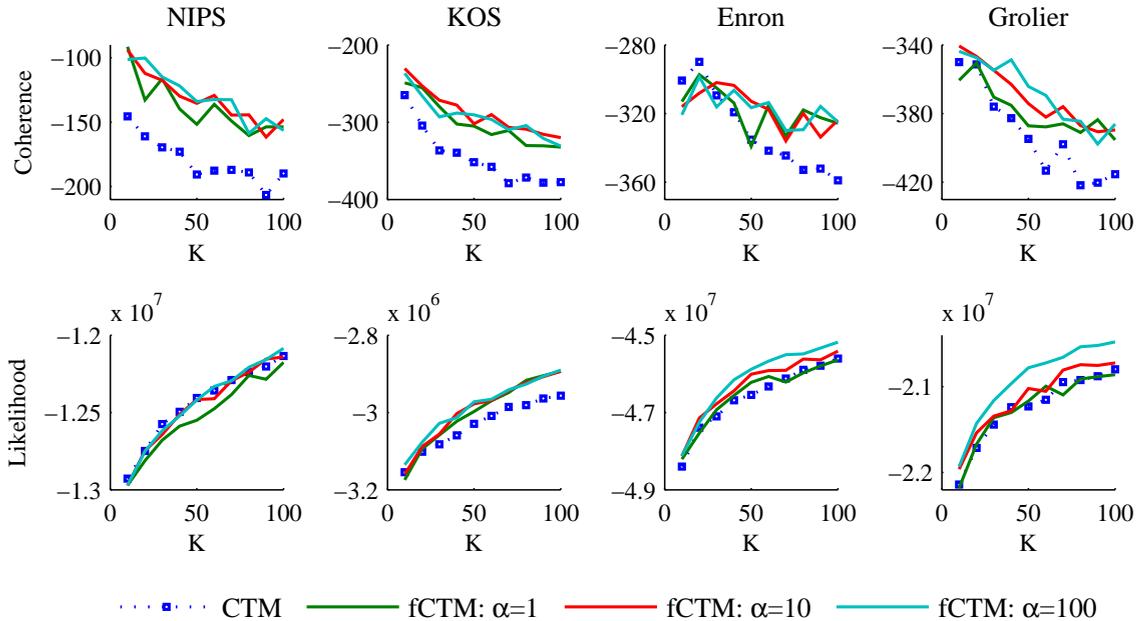}\\
  \caption[Performance of fCTM and CTM]{Quality of  the models which were learned by fCTM (solid lines) and CTM (dashed lines). Higher is better.}
  \label{ch5-fig:quality-fCTM-CTM}
\end{figure}

Figure \ref{ch5-fig:quality-fCTM-CTM} shows the quality of the learned models. We observe that the two learning methods performed comparably in terms of likelihood. Note from Figure~\ref{ch5-fig:time-fCTM-CTM} that fCTM is able to reach comparable likelihood to CTM within few iterations, even though the main objective function of fCTM for learning is not likelihood. This behavior shows further the advantage of our algorithm. 

In terms of coherence, the topic quality of CTM seems to be inferior to that of fCTM. Both methods often tend to learn less interpretable (but more specific) topics as the number $K$ increases. CTM seems to degrade topic quality faster than fCTM as increasing $K$. We observe further that fCTM often learns significantly better topics than CTM in the cases of large $K$. When investigating the models learned by fCTM, we find that individual topics are very meaningful as depicted partially in  Figure~\ref{ch5-fig:pos-correlation-gro-100-fCTM}. Those observations demonstrate advantages of fCTM  over CTM for practical applications, such as exploration or discovery of interactions of hidden topics/factors.

\textit{Models of hidden interactions:} Figure~\ref{ch5-fig:pos-correlation-gro-100-fCTM} and \ref{ch5-fig:neg-correlation-gro-100-fCTM} shows  parts of the full model with 100 topics learned by fCTM from Grolier. Figure \ref{ch5-fig:pos-correlation-gro-100-fCTM} shows positive correlations between topics, while Figure \ref{ch5-fig:neg-correlation-gro-100-fCTM} shows negative correlations. It can be observed that the learned topics are interpretable and the discovered correlations are reasonable. Those further support that fCTM is able to learn qualitative models.

\begin{figure}
\centering
  \includegraphics[trim=29cm 5cm 39cm 10cm, clip=true, width=\textwidth] {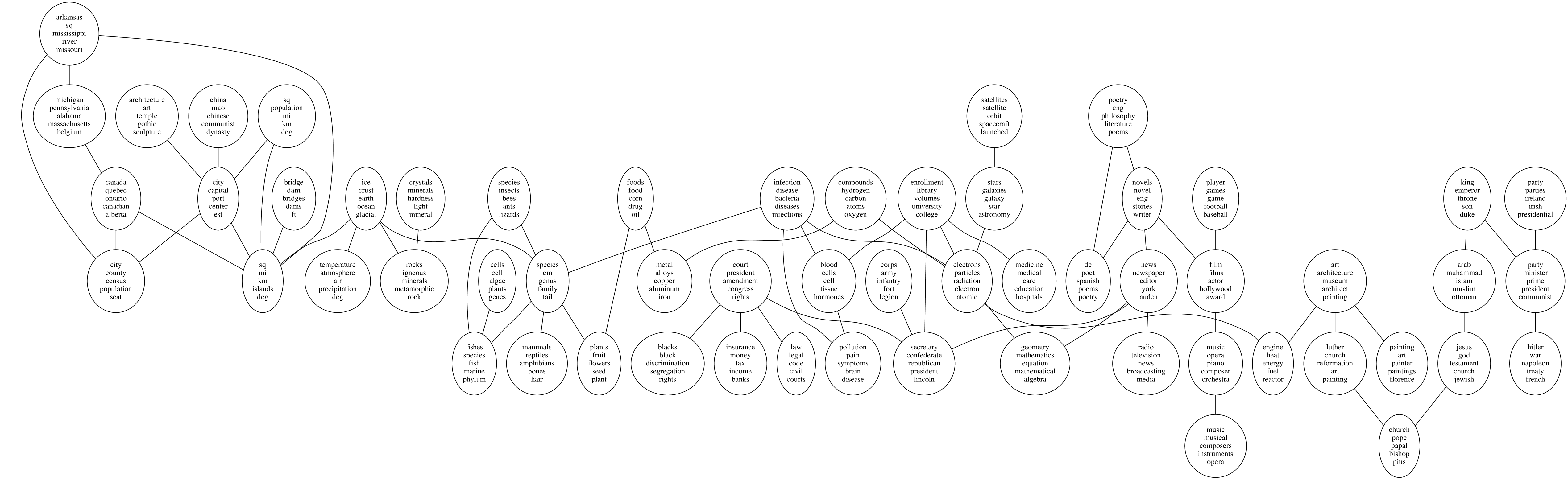}\\
  \caption[Model: topics with positive correlations]{Illustration of the correlated topic model with 100 topics which was learned by fCTM from Grolier articles. An edge connecting two topics shows that if one topic appears  in a document, the other likely appears as well. This visualization was drawn with Graphviz \citep{GansnerN00graphviz}}.
  \label{ch5-fig:pos-correlation-gro-100-fCTM}
\end{figure}

\begin{figure}
\centering
  \includegraphics[width=\textwidth] {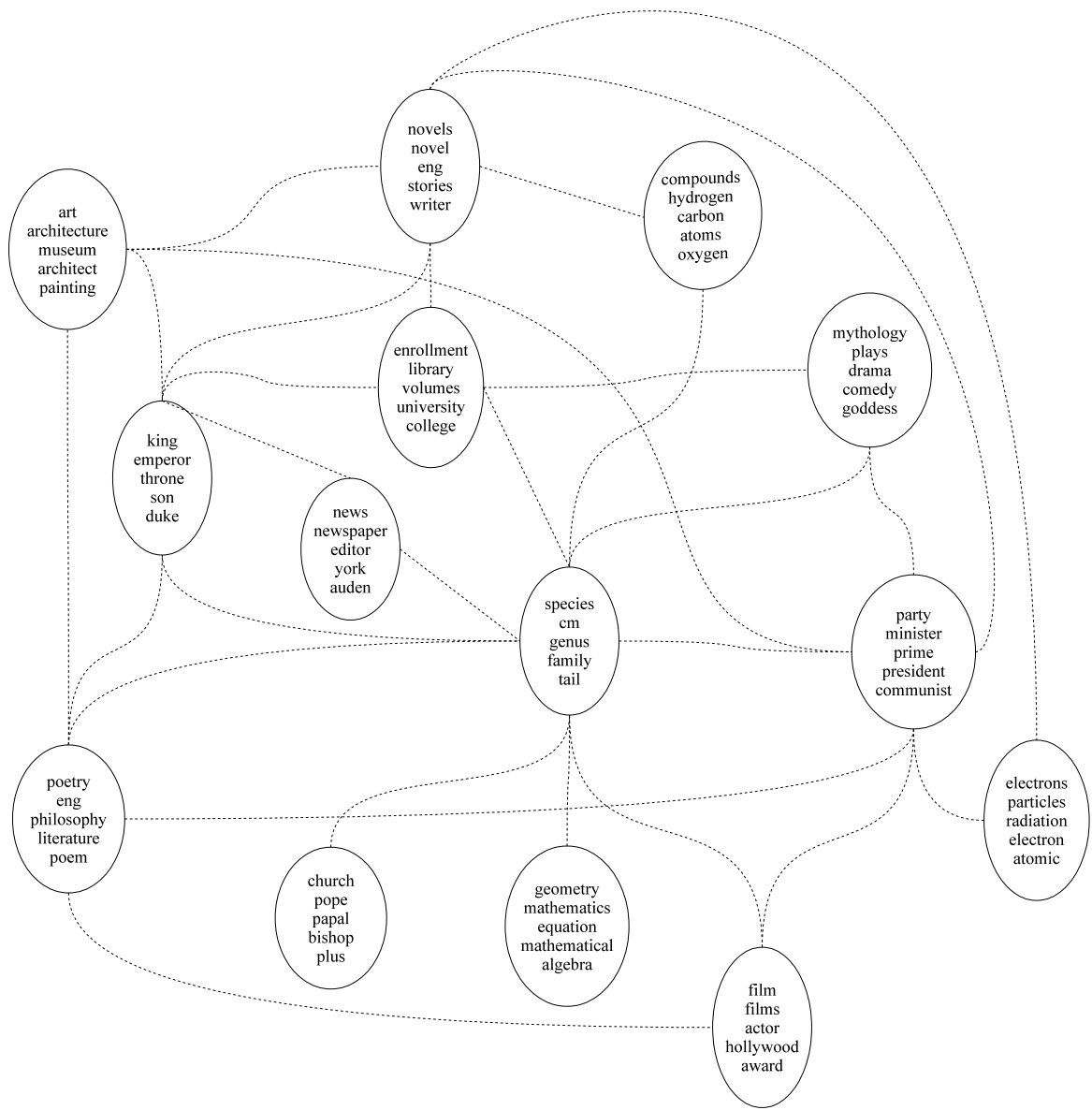}\\
  \caption[Model: topics with negative correlations]{Illustration of the correlated topic model with 100 topics which was  learned by fCTM from Grolier. An edge connecting two topics shows that the two topics \textit{unlikely appear together} in a document. }
  \label{ch5-fig:neg-correlation-gro-100-fCTM}
%\vskip -1in
\end{figure}

%\begin{landscape}
%\begin{figure}
%\centering
%  \includegraphics[trim=19.5cm 5cm 40cm 8cm, clip=true, totalheight=0.6\textheight] {gro-a100--100-fCTM-C003-pos-correlation-dot.pdf}\\
%  \caption[Model: topics with positive correlations]{Illustration of the correlated topic model with 100 topics which was learned by fCTM from Grolier articles. An edge connecting two topics shows that the two topics \textit{very likely appear together}  in a document. This visualization was drawn with Graphviz \citep{GansnerN00graphviz}}.
%  \label{ch5-fig:pos-correlation-gro-100-fCTM}
%\end{figure}
%\end{landscape}

\subsubsection{Quality and speed of OFW}

We have seen in the last parts that OFW (an example of SGD algorithms) is really beneficial in helping fCTM to work efficiently. It seems to have more advantages than variational methods when being employed in CTM. Next, we are interested in performance of OFW as an algorithm for non-concave problems. %Particularly, quality, speed, and stability of OFW would be our focus.

Problem~(\ref{ch5-eq:003}) was used for investigation. The testing parts of the  datasets were used to provide documents ($\mbf{d}$) for (\ref{ch5-eq:003}). We used the models ($\Upsilon = \{\mbf{\beta}, \mbf{\mu}, \mbf{\Sigma}\}$) which have 100 topics and have been learned previously from the training data. Totally, we have 11197 instances of problem~(\ref{ch5-eq:003}) for investigation.

For comparison, we took \textit{Sequential Least Squares Programming} (SLSQP) as a standard method for non-convex optimization \citep{PerezJM09-pyOpt}. Various methods have been proposed, but SLSQP seems to be one among the best solvers according to different tests \citep{PerezJM09-pyOpt}. Therefore it was taken in comparison with OFW.\footnote{The variational method by \cite{BleiL07} was not considered for comparison. The reason comes from the differrence of problems to be solved. Indeed, OFW tries to maximize $\Pr(\mbf{\theta, d})$ whereas the variational method tries to maximize a lower bound of the likelihood of document $\mbf{d}$. Hence it is difficult to compare quality of the two methods. The last subsection has discussed inference time of the two methods.}
The same criterion was used to assess convergence for both methods: relative improvement of objective functions is no better than $10^{-6}$, and the number of iterations is at most 100.

\begin{table}[tbp]
\caption{Statistics of OFW and SLSQP after solving 11197 non-concave problems. ``{\#Fails to solve}" shows the number of problems that a method found infeasible solutions. The last three rows show the number of problems that a method performs better than ($>$) or comparably with ($\approx$) or worse than ($<$) the other one. We observe that OFW often performs 150-2000 times faster than SLSQP. For most problems of interest, OFW found significantly better solutions than SLSQP.}
\label{fig:-OFW-inf-compare}
\begin{center}
\begin{tabular}{|l|r|r|r|r|r|}
\hline 
\multicolumn{ 2}{|l|}{\textbf{Data}} & \textbf{ NIPS} & \textbf{ KOS} & \textbf{ Enron} & \textbf{ Grolier } \\ \hline
\multicolumn{ 2}{|l|}{Total number of problems} & 150 &  343 &  3986 &  6718  \\ \hline \hline
\multicolumn{ 1}{|l|}{Average time (seconds)} & SLSQP & 18.6729 & 2.2018 & 1.3871 & 1.5579 \\ \cline{ 2- 6}
\multicolumn{ 1}{|l|}{ to solve a problem} & \textit{OFW} & \textit{0.0161} & \textit{0.0069} & \textit{0.0056} & \textit{0.0073} \\ \hline \hline
\multicolumn{ 1}{|l|}{Objective value (averaged)} & SLSQP & -9708.9410 & -169.3041 & 117.7672 & -127.7626 \\ \cline{ 2- 6}
\multicolumn{ 1}{|l|}{} & \textit{OFW} & \textit{-8860.8170} & \textit{1464.7165} & \textit{1753.9598} & \textit{1572.2495} \\ \hline \hline
\multicolumn{ 1}{|l|}{\#Fails to solve} & SLSQP & 37 & 172 & 1981 & 3316 \\ \cline{ 2- 6}
\multicolumn{ 1}{|l|}{} & \textit{OFW} & \textit{0} & \textit{0} & \textit{0} & \textit{0} \\ \hline \hline
\multicolumn{2}{|l|}{\#OFW $>$ SLSQP} & 129 & 343 & 3979 & 6700 \\ \hline
\multicolumn{2}{|l|}{\#OFW $\approx$ SLSQP} & 21 & 0 & 7 & 16 \\ \hline
\multicolumn{2}{|l|}{\#OFW $<$ SLSQP} & 0 & 0 & 0 & 2 \\ \hline 
\end{tabular}
\end{center}
\end{table}

Table \ref{fig:-OFW-inf-compare} shows some statistics from our experiments. It can be observed that SLSQP often needs intensive time to solve a problem, while OFW consumes substantially less time. We observe that OFW often works 150-2000 times faster than SLSQP. Slow performance of SLSQP mainly comes from the need to solve many intermediate quadratic programming problems, each of which often requires considerable time in our observations. On contrary, each iteration of OFW is very modest,  which mostly requires computation of partial derivatives. 

In terms of quality, we observe that OFW was able to find significantly better approximate solutions than SLSQP. When inspecting individual problems, we found that SLSQP failed to find feasible solutions for many problems, e.g., a large number of returned solutions were significantly out of domain ($\overline{\Delta}_K$). In contrast, OFW always manages to find feasible solutions. Among 11197 problems, OFW performed significantly worse than SLSQP for only 2. Those observations demonstrate that OFW has many advantages over  (deterministic) SLSQP. Further, it is able to find good approximation solutions for non-concave problems with a modest requirement of computation. 

%\textit{Stability:}
%We have seen good behaviors of OFW in solving non-concave problems. However, due to its stochastic nature, one  question  arises naturally relating to stability of the algorithm. That is, \textit{for a given problem, does the algorithm output  comparable solutions from different repetitions?} To answer this question, we chose 6718 inference problems which associate with the testing part of Grolier. 100 random runs were carried out for each problem. Figure~\ref{fig:-stability-OFW-100rep} shows some example problems for which OFW performed most stably (or unstably). It can be observed for many   problems the algorithm returned very stable results after 100 different runs.
%
%
%\begin{figure}
%\centering
%  \includegraphics[width=\textwidth]{stability-OFW-100rep.pdf}\\
%  \caption{Stability of OFW in 100 random runs, investigated from 6718 inference problems of Grolier. The first two columns show 10 problems for which OFW performed  most stably; and the last two columns show 10 problems for which OFW performed  most unstably. `*' indicates the average objective value from 100 runs.}
%  \label{fig:-stability-OFW-100rep}
%\end{figure}

\section{Conclusion and discussion} \label{ch5-sec:conclusion}

We have introduced  the concept of probable convexity to analyze real functions or families of functions. It is the way to see how probable a real function is convex. In particular, it can reveals how many members of a family of functions are convex. When a family contains most convex members,  we could deal with the family efficiently in practice. Hence probable convexity provides a feasible way to deal with non-convexity of real problems such as  posterior estimation  in  probabilistic graphical models.

When analysing probable convexity of the problem of estimating topic mixtures in CTM \citep{BleiL07}, we found that this problem is concave under certain conditions.  The same results were obtained for many nonconjugate models. These results suggest that posterior inference of topic mixtures in those models might be done efficiently in practice, which seems to contradict with the belief of intractability in the literature. Benefiting from those theoretical results,  we proposed a novel algorithm for learning CTM which can work 60-170 times faster than the variational method by \cite{BleiL07}, while keeping  or making better the quality of the learned models. We believe that by using the same methodology as ours, learning for many existing nonconjugate models can be significantly accelerated. An implementation of our algorithm is freely available at \url{http://is.hust.edu.vn/~khoattq/codes/fCTM/}

There is a unusual employment of the Online Frank-Wolfe algorithm (OFW) \citep{Hazan2012OFW} to solve nonconvex problems (inference of topic mixtures in CTM). OFW is a specific instance of stochastic gradient descent algorithms (SGDs) for solving convex problems. By a careful employment, OFW behaves well in solving the inference problem which is nonconcave in the worse case. It helps us to design an efficient and effective algorithm for learning CTM. Such a successful use of OFW suggests that SGDs might be a practical choice to deal with nonconvex problems. In our experiments, OFW found significantly better solutions whereas performed 150-2000 times faster than SLSQP (the standard algorithm for nonconvex optimization). This further supports our highlight about SGDs. We hope that this highlight would open various rooms for future studies on connection of SGDs with nonconvex optimization. 

%\acks{K. Than is supported by a MEXT scholarship, Japan.}

\pagebreak

%\bibliography{topic-models-all,other-all}

\end{document}